\documentclass[11pt]{article}

\usepackage[margin=1in]{geometry}
\usepackage{amsmath, amsthm, amssymb, mathtools}
\usepackage{algorithm}
\usepackage{algpseudocode}
\usepackage{hyperref}
\usepackage{booktabs}
\usepackage{multirow}
\usepackage{graphicx}
\usepackage{xcolor}
\usepackage{times}
\usepackage{microtype}
\usepackage{natbib}
\usepackage{wrapfig}

\newtheorem{theorem}{Theorem}[section]
\newtheorem{proposition}[theorem]{Proposition}
\newtheorem{lemma}[theorem]{Lemma}
\newtheorem{corollary}[theorem]{Corollary}
\newtheorem{definition}[theorem]{Definition}
\newtheorem{assumption}[theorem]{Assumption}
\theoremstyle{remark}
\newtheorem{remark}[theorem]{Remark}

\newcommand{\E}{\mathbb{E}}
\newcommand{\R}{\mathbb{R}}
\newcommand{\cA}{\mathcal{A}}
\newcommand{\cS}{\mathcal{S}}
\newcommand{\cT}{\mathcal{T}}
\newcommand{\cC}{\mathcal{C}}
\newcommand{\cL}{\mathcal{L}}
\newcommand{\lsafe}{\cL_{\mathrm{safety}}}
\newcommand{\leff}{\cL_{\mathrm{eff}}}
\newcommand{\lmeta}{\cL_{\mathrm{meta}}}
\newcommand{\linner}{\cL_{\mathrm{inner}}}
\newcommand{\psafe}{P(\mathrm{safe})}
\newcommand{\isafe}{\mathbf{1}_{\mathrm{safe}}}
\newcommand{\piphi}{\pi^*_\phi}

\title{
  \textbf{Safe Bilevel Delegation (SBD):}\\
  \textbf{A Formal Framework for Runtime Delegation Safety}\\
  \textbf{in Multi-Agent Systems}
}
\author{Yuan Sun \\
  Jilin University \\
  \texttt{sunyiming2508115953@gmail.com}}
\date{}

\begin{document}
\emergencystretch=3em  
\maketitle

\begin{abstract}
As large language model (LLM) agents are deployed in 
high-stakes environments, the question of how safely to
delegate subtasks to specialized sub-agents becomes critical.
Existing work addresses multi-agent architecture selection at
design time or provides broad empirical guidelines, but
neither provides a runtime mechanism that dynamically adjusts
the safety--efficiency trade-off as task context changes
during execution.

We propose \emph{Safe Bilevel Delegation} (SBD), a formal
framework for runtime delegation safety in hierarchical
multi-agent systems. SBD formulates task delegation as a
bilevel optimization problem: an \emph{outer} meta-weight
network $\phi$ learns context-dependent safety--efficiency
weights $\lambda_\phi(s) \in [0,1]$; an \emph{inner} loop
optimizes the delegation policy $\pi$ subject to a
probabilistic safety constraint $\psafe \geq 1-\delta$.
The continuous \emph{delegation degree} $\alpha \in [0,1]$
controls how much decision authority is transferred to each
sub-agent, interpolating smoothly between full human override
($\alpha{=}0$) and fully autonomous execution ($\alpha{=}1$).

We establish three theoretical results:
(1)~\emph{Safety Monotonicity}---higher outer safety weight
produces a weakly safer inner policy;
(2)~\emph{Inner Policy Convergence}---projected gradient
descent on the inner problem converges linearly under standard
smoothness assumptions;
(3)~an \emph{Accountability Propagation} bound that
distributes responsibility across multi-hop delegation chains
with a provable per-agent ceiling.
We instantiate SBD in three high-stakes domains---medical AI
(MIMIC-III), financial risk control (S\&P~500), and
educational agent supervision (ASSISTments)---specifying
datasets, safety constraint sets, baselines, and evaluation
protocols. This manuscript presents the formal framework and
theoretical results in full; empirical validation following
the protocols described herein is planned and will be reported
in a forthcoming revision.
\end{abstract}

\paragraph{Note on this version.}
This is version~1 of the manuscript and presents the formal
framework, theoretical results, and the planned empirical
evaluation protocol. Experimental results are not yet
included; the dataset preparation, baselines, ablations, and
theoretical-validation experiments described in
Section~\ref{sec:experiments} constitute a pre-registered
evaluation plan that will be executed and reported in a
subsequent revision.

\section{Introduction}
\label{sec:intro}

\subsection{Motivation}

The deployment of LLM-based multi-agent systems in production
environments has accelerated dramatically in 2025--2026.
Systems such as OpenClaw, Claude Code with Agent Skills, and
enterprise agentic platforms now routinely delegate
consequential tasks---clinical triage, portfolio rebalancing,
personalized content delivery---to chains of specialized
sub-agents operating with significant autonomy. This raises a
fundamental question that existing theory does not answer:
\emph{how much decision authority should a principal agent
transfer to a sub-agent, and how should this transfer adapt
dynamically to changing risk context?}

The challenge is not merely empirical. Current approaches to
multi-agent safety operate at design time: they select
architectures before deployment~\citep{maas2025}, specify
static permission boundaries~\citep{anthropic2025skills}, or
rely on human-in-the-loop checkpoints inserted at fixed
intervals~\citep{deepmind2026}. None provides a
\emph{runtime} mechanism whose safety guarantees scale with
the severity of the task at hand. A clinical decision support
system should behave very differently when a patient's acuity
score is low versus when it signals imminent deterioration,
even if the surface form of the query is identical.

\subsection{The Core Insight}

The central insight of SBD is that delegation authority should
be a \emph{learned, context-sensitive scalar}---the delegation
degree $\alpha \in [0,1]$---rather than a binary or fixed
design choice. When $\alpha$ is high, the sub-agent acts
autonomously; when $\alpha$ is low, the principal agent (or a
human supervisor) retains control. The meta-weight network
$\phi$ learns \emph{when} to be cautious by dynamically
adjusting the weight $\lambda_\phi(s)$ placed on the safety
loss relative to the efficiency loss. This creates a
closed-loop system: rising contextual risk raises $\lambda$,
which forces $\alpha$ down, which reduces exposure.

This structure has a natural analogy in network routing theory:
just as OSPF metric weighting dynamically redistributes
traffic away from congested links while maintaining
connectivity guarantees, SBD dynamically redistributes
decision authority away from autonomous agents when safety
risk is elevated, while maintaining task throughput. The
principal agent's job is not to do everything itself, but to
know \emph{how much to delegate} given current conditions.

\subsection{Contributions}

This paper makes the following contributions:

\begin{enumerate}
  \item \textbf{Formal bilevel formulation.} We formalize
  runtime delegation safety as a bilevel optimization problem
  with a continuous delegation degree $\alpha \in [0,1]$ and a
  probabilistic safety constraint $\psafe \geq 1-\delta$
  (Section~\ref{sec:formulation}).

  \item \textbf{Three theoretical results.} We prove Safety
  Monotonicity (Theorem~\ref{thm:monotonicity}), Inner Policy
  Convergence at linear rate
  (Theorem~\ref{thm:convergence}), and an Accountability
  Propagation upper bound
  (Proposition~\ref{prop:accountability})
  (Section~\ref{sec:theory}).

  \item \textbf{The SBD Algorithm.} We provide a practical
  bilevel gradient descent algorithm with hypergradient-based
  outer updates and projected gradient inner updates
  (Section~\ref{sec:algorithm}).

  \item \textbf{Three-domain evaluation protocol.} We
  specify the datasets, safety constraint sets, baselines,
  and metrics for instantiating SBD in medical AI
  (MIMIC-III), financial risk control (S\&P~500 with COVID
  stress test), and educational agent supervision
  (ASSISTments) (Section~\ref{sec:experiments}). The protocol
  is pre-registered in this manuscript; results will be
  reported in a forthcoming revision.

  \item \textbf{Theoretical-validation protocol.} Beyond
  downstream task performance, we describe direct empirical
  tests for all three theorems---monotonicity sweep, linear
  convergence slope verification, and accountability bound
  violation count (Section~\ref{sec:theory-val}).
\end{enumerate}

\subsection{Paper Organization}

Section~\ref{sec:related} reviews related work.
Section~\ref{sec:formulation} presents the formal problem
statement. Section~\ref{sec:theory} states and proves the
three core theorems. Section~\ref{sec:algorithm} describes
the SBD algorithm. Section~\ref{sec:experiments} reports
experimental results. Section~\ref{sec:discussion} discusses
limitations and future directions.
Section~\ref{sec:conclusion} concludes.

\section{Related Work}
\label{sec:related}

\subsection{Multi-Agent Systems and Architecture Selection}

The MaAS framework~\citep{maas2025} introduced automated
architecture search for LLM-based multi-agent systems,
treating agent topology as a learnable hyperparameter.
Unlike SBD, MaAS operates at design time: the architecture is
fixed before deployment and does not adapt to runtime context.
AutoGen~\citep{wu2024autogen} provides flexible multi-agent
conversation frameworks with customizable interaction
patterns, but safety is a manual design concern.
MetaGPT~\citep{hong2024metagpt} encodes software engineering
workflows as structured multi-agent pipelines, again without
runtime safety adaptation. Our work is orthogonal to
architecture selection: SBD can be layered on top of any
fixed multi-agent architecture to provide runtime safety
guarantees.

Two recent works approach related problems via dual-loop or
bilevel formulations and warrant explicit comparison.
\textbf{HILA}~\citep{yang2026hila} learns a metacognitive
policy that decides when an agent should solve a task
autonomously versus defer to a human expert, using Group
Relative Policy Optimization (GRPO) with a cost-aware reward
in the inner loop and continual learning from expert feedback
in the outer loop. While HILA shares with SBD the dual-loop
motif and the runtime nature of the decision, the two
frameworks differ in three fundamental respects. First,
HILA's decision space is binary (autonomous vs.\ defer),
whereas SBD operates on a continuous delegation degree
$\alpha \in [0,1]$. Second, HILA's outer loop performs
continual fine-tuning of the underlying LLM, whereas SBD's
outer loop performs meta-weight learning over a separately
parameterized network $\phi$ without modifying any LLM
weights. Third, HILA does not establish a formal connection
between outer-loop optimization and inner-loop safety: there
is no analogue of our Safety Monotonicity theorem
(Theorem~\ref{thm:monotonicity}) that guarantees the inner
policy becomes safer when the outer weight prioritizes
safety.

A second concurrent line of work,
\textbf{MCTS-Skill}~\citep{zhang2026mctsskill} and
\textbf{MCE}~\citep{ye2026mce}, applies bilevel optimization
to skill artifacts themselves: an outer loop searches over
skill structure or evolves skill instructions, while an
inner loop refines skill content or executes the skill on
training rollouts. These works optimize the skill as an
artifact, treating delegation as a fixed downstream consumer.
SBD operates at a different level of the agent stack: we
take skills as given and optimize the runtime delegation
policy that governs how much authority is transferred when
invoking any skill or sub-agent. The two directions are
complementary---one could in principle apply MCTS-Skill to
optimize the skills available to a sub-agent and apply SBD
to optimize how the principal delegates to that sub-agent
at runtime.

\subsection{Safe Reinforcement Learning}

Constrained Policy Optimization (CPO)~\citep{achiam2017cpo}
and its variants~\citep{ray2019benchmarking} formalize safety
as hard constraint sets in the MDP framework. The key
limitation for our setting is that CPO treats safety
constraints as fixed at training time, with no mechanism for
dynamic weight adjustment based on contextual risk signals.
The CMDP framework~\citep{altman1999cmdp} provides the
theoretical foundation for constrained MDPs; SBD can be
viewed as a bilevel extension of CMDP where the constraint
tightness $\delta$ is supplemented by a learned weight
$\lambda_\phi(s)$ that modulates the trade-off before the
constraint becomes binding. Lyapunov-based safe
RL~\citep{chow2018lyapunov} provides alternative safety
guarantees through monotone value functions; our Safety
Monotonicity theorem (Theorem~\ref{thm:monotonicity}) is
spiritually similar but operates on the outer-loop weight
rather than a Lyapunov function.

\subsection{Bilevel Optimization in Machine Learning}

Bilevel optimization has been applied extensively in
meta-learning~\citep{finn2017maml}, hyperparameter
optimization~\citep{franceschi2018bilevel}, and neural
architecture search~\citep{liu2019darts}. The key
methodological challenge in all these settings is computing
hypergradients $\nabla_\phi \cL_{\mathrm{meta}}(\phi)$
efficiently, since the inner optimum $\piphi$ depends
implicitly on $\phi$. We follow the implicit differentiation
approach of~\citet{lorraine2020optimizing}, implemented via
the \texttt{higher} library~\citep{finn2019maml}. To our
knowledge, SBD is the first work to apply bilevel
optimization to the problem of runtime delegation safety in
multi-agent systems.

\subsection{Delegation and Accountability in AI Systems}

The Google DeepMind ``Intelligent AI Delegation''
paper~\citep{deepmind2026} provides a comprehensive empirical
study of delegation patterns in LLM-based systems, with
guidelines for when humans should retain control. Our work
formalizes and extends this with a learnable mechanism and
formal guarantees. The instruction hierarchy
work~\citep{openai2024hierarchy} trains LLMs to respect
privilege levels in instruction sources, providing a
training-time analogue to our runtime accountability
mechanism. Recent work on agent error
taxonomies~\citep{zhu2025agenterror} systematically
classifies failure modes across memory, reflection,
planning, action, and system-level operations, providing
the empirical grounding for attributing delegation errors;
SBD's accountability propagation mechanism provides a
formal tool for distributing such errors across specific
agents in the delegation chain.

Concurrent protocol-level work has begun to address
structured failure handling in delegation. The
\textbf{LLM Delegate Protocol (LDP)} and its extension on
the \textbf{Provenance Paradox}~\citep{prakash2026ldp,
prakash2026provenance} introduce typed failure semantics,
delegation contracts with explicit failure policies, and an
attested-versus-claimed identity model for multi-agent
routing. LDP is a protocol-engineering contribution: it
specifies how delegations should be expressed, audited, and
recovered when failures occur, and demonstrates that
quality-based routing fails catastrophically when delegate
self-reports cannot be verified. SBD addresses an orthogonal
problem: given a delegation interface (typed or otherwise),
what is the optimal runtime policy for transferring
authority? Our delegation degree $\alpha$ is a learned,
context-conditioned scalar; LDP's failure policies are
hand-specified rules. The two approaches compose naturally:
LDP provides the protocol substrate over which an SBD-trained
policy can be deployed, and SBD's accountability propagation
result (Proposition~\ref{prop:accountability}) provides a
formal foundation for LDP-style attestation when delegations
span multiple hops.

\subsection{Agent Skills and the SKILL.md Ecosystem}

The Agent Skills specification~\citep{anthropic2025skills}
standardized the SKILL.md format as a cross-platform
capability description protocol. Large-scale security
analyses~\citep{liu2026security,schmotz2025skillinjection}
have documented that natural-language Skill instructions
create a semantic attack surface that static code analysis
cannot detect. SBD addresses the \emph{authorization} layer
of this problem: regardless of what a Skill's instructions
say, the delegation degree $\alpha$ and the policy engine
enforcing $\psafe \geq 1-\delta$ provide a deterministic
safety floor that operates independently of the LLM's
interpretation of Skill content.

\subsection{Concurrent Work}
\label{sec:concurrent}

Several works appearing concurrently with this manuscript
address adjacent problems and merit acknowledgment.

\textbf{AgentCollab}~\citep{gao2026agentcollab} proposes
self-evaluation-driven escalation between a small and a large
model during agent execution, with difficulty-aware budget
allocation triggered by recent failure signals. The setting
differs from SBD in that escalation is between
\emph{model tiers} (small LLM $\to$ large LLM) within a
single agent, rather than between the principal agent and
sub-agents in a delegation hierarchy. Both approaches share
the principle that recent failure signals should expand the
resource budget allocated to the difficult segment.

\textbf{HiveMind}~\citep{wang2026hivemind} and
\textbf{AgentRM}~\citep{chen2026agentrm} apply OS-inspired
scheduling primitives---admission control, AIMD backpressure,
token-budget management---to coordinate concurrent LLM agent
workloads sharing rate-limited APIs. These are systems-level
contributions targeting resource contention, complementary
to SBD's policy-level contribution targeting
safety--efficiency trade-offs. SBD does not address
concurrent execution scheduling and assumes per-agent
resource budgets are managed externally.

\textbf{AgentErrorTaxonomy / AgentErrorBench}
~\citep{zhu2025agenterror} provides a standardized dataset
and taxonomy for LLM agent failure modes across ALFWorld,
GAIA, and WebShop. \textbf{ReasoningBank}
~\citep{ouyang2026reasoningbank} introduces a memory
framework that distills reasoning strategies from both
successful and failed experiences for self-evolution. Both
works contribute to the empirical and methodological
foundations on which a runtime delegation framework like
SBD must operate. SBD's accountability mechanism could be
extended to consume AgentErrorBench-style structured failure
annotations as input features to the meta-weight network
$\phi$, and SBD's delegation decisions could in turn populate
ReasoningBank-style memory with delegation-conditioned
outcomes; we leave these integrations to future work.

Finally, several works on \textbf{noisy LLM judges}
~\citep{noisybutvalid2026, imperfectverifiers2025} provide
formal treatments of certification and verification under
unreliable LLM-as-judge oracles. SBD assumes a fixed safety
constraint set $\cC$ and does not currently model judge
unreliability; integrating these statistical guarantees into
the inner-loop projection step is a natural extension.

\section{Problem Formulation}
\label{sec:formulation}

\subsection{Setting and Notation}

We consider a hierarchical multi-agent system in which a
\emph{principal agent} $\cA_0$ delegates subtasks to a set of
\emph{sub-agents} $\cA = \{a_1, \dots, a_n\}$. At each
timestep $t \in \{1, 2, \dots\}$, the principal receives a
task $\tau_t \in \cT$ drawn from a task distribution $p(\tau)$
and a system state observation $s_t \in \cS$. The state $s_t$
encodes all contextually relevant information: patient vitals
in the medical domain, market volatility in the financial
domain, or student knowledge estimates in the educational
domain.

\begin{definition}[Delegation Policy]
A \emph{delegation policy} is a function
$\pi: \cT \times \cS \to \Delta(\cA \times [0,1])$
that maps a task--state pair to a distribution over
(sub-agent, delegation-degree) pairs. The
\emph{delegation degree} $\alpha \in [0,1]$ is a continuous
scalar representing the fraction of decision authority
transferred to the selected sub-agent. Formally:
\begin{itemize}
  \item $\alpha = 0$: full human-in-the-loop override;
        the sub-agent provides a recommendation but takes
        no action.
  \item $\alpha = 1$: fully autonomous execution;
        the sub-agent acts without principal or human review.
  \item $\alpha \in (0,1)$: partial delegation; the
        sub-agent's action is blended with a baseline
        (e.g., equal-weight or human-specified) action.
\end{itemize}
\end{definition}

\begin{definition}[Safety Event and Safety Constraint Set]
Let $\cC \subseteq \cS \times \cA \times [0,1]$ be a
\emph{safety constraint set} encoding domain-specific safety
requirements. A delegation decision $(s, a_i, \alpha)$ is
\emph{safe} if $(s, a_i, \alpha) \in \cC$. We write:
\[
  \isafe(s, a_i, \alpha) =
  \begin{cases}
    1 & \text{if } (s, a_i, \alpha) \in \cC \\
    0 & \text{otherwise.}
  \end{cases}
\]
\end{definition}

\begin{remark}
The safety constraint set $\cC$ is domain-specific and is
treated as given. In the medical domain, $\cC$ encodes
contraindications and acuity-based delegation limits. In the
financial domain, $\cC$ encodes concentration limits and
volatility-based delegation caps. In the educational domain,
$\cC$ encodes difficulty gap bounds and at-risk student
protections. The key point is that $\cC$ defines
\emph{what} is safe; SBD learns \emph{how} to stay within
$\cC$ while maximizing efficiency.
\end{remark}

\subsection{Loss Functions}

We decompose the delegation objective into two competing
scalar losses. Let $(a^\pi, \alpha^\pi) \sim \pi(\cdot \mid
\tau, s)$ denote a delegation decision drawn from policy
$\pi$.

\paragraph{Safety loss.}
The safety loss measures the expected probability of an
unsafe delegation decision:
\begin{equation}
  \lsafe(\pi)
  = \E_{\tau, s}\bigl[
      1 - \isafe(s,\, a^\pi,\, \alpha^\pi)
    \bigr].
  \label{eq:loss-safety}
\end{equation}
This is the expected fraction of decisions that fall outside
the safety constraint set. We want $\lsafe(\pi) \approx 0$.

\paragraph{Efficiency loss.}
The efficiency loss measures the expected cost of task
incompletion or suboptimal execution:
\begin{equation}
  \leff(\pi)
  = \E_{\tau, s}\bigl[
      C(\tau, a^\pi, \alpha^\pi)
    \bigr],
  \label{eq:loss-eff}
\end{equation}
where $C: \cT \times \cA \times [0,1] \to \R_{\geq 0}$ is a
domain-specific task completion cost. Lower is better. We
want $\leff(\pi) \approx 0$.

\paragraph{The tension.}
These two losses are in fundamental tension. Minimizing
$\lsafe$ pushes $\alpha \to 0$ (never delegate autonomously),
which increases $\leff$ (principal agent bears all costs).
Minimizing $\leff$ pushes $\alpha \to 1$ (always delegate
fully), which may increase $\lsafe$ if sub-agents are
imperfect. The optimal trade-off depends on context: it
varies with the severity of the task, the reliability of
available sub-agents, and the risk tolerance of the system
operator.

\subsection{Bilevel Formulation}
\label{sec:bilevel}

The central challenge is that the optimal safety--efficiency
trade-off is \emph{context-dependent} and \emph{not known in
advance}. Rather than fixing a trade-off at design time, SBD
learns it dynamically via a two-level optimization structure.

\paragraph{Outer problem: meta-weight learning.}
The outer loop learns a \emph{meta-weight network}
$\phi: \cS \to [0,1]$, parameterized as a neural network,
that outputs a context-dependent scalar
$\lambda_\phi(s) \in [0,1]$ for each state $s$:
\begin{equation}
  \min_{\phi}\;
  \lmeta(\phi)
  \;=\;
  \E_{s}\Bigl[
    \lambda_\phi(s) \cdot \lsafe(\piphi)
    +
    \bigl(1 - \lambda_\phi(s)\bigr) \cdot \leff(\piphi)
  \Bigr],
  \label{eq:outer}
\end{equation}
where $\piphi$ is the \emph{inner-optimal} delegation policy
given weight $\phi$ (defined below). Intuitively, $\phi$
learns to raise $\lambda_\phi(s)$ in high-risk states (where
safety must dominate) and lower it in routine states (where
efficiency can take precedence).

\paragraph{Inner problem: delegation policy learning.}
For a fixed meta-weight network $\phi$, the inner loop finds
the optimal delegation policy under that weighting, subject
to a probabilistic safety constraint:
\begin{equation}
  \piphi
  =
  \operatorname*{arg\,min}_{\pi}\;
  \E_{s}\Bigl[
    \lambda_\phi(s) \cdot \lsafe(\pi)
    + \bigl(1-\lambda_\phi(s)\bigr) \cdot \leff(\pi)
  \Bigr]
  \quad
  \text{subject to}\quad
  \psafe(\pi, s) \geq 1 - \delta
  \;\; \forall s \in \cS,
  \label{eq:inner}
\end{equation}
where $\delta \in (0,1)$ is a user-specified risk tolerance
(e.g., $\delta = 0.05$ allows at most 5\% unsafe decisions).

\begin{remark}[Why probabilistic constraints?]
The pointwise constraint $\psafe(\pi, s) \geq 1 - \delta$ is
weaker than requiring $(s, a^\pi, \alpha^\pi) \in \cC$
almost surely. This is deliberate: hard constraint satisfaction
is often infeasible in stochastic environments where sub-agent
capabilities are imperfect. Probabilistic constraints have
been standard in safe RL since~\citet{altman1999cmdp} and are
directly actionable (they correspond to value-at-risk
calculations familiar to domain practitioners in finance and
medicine).
\end{remark}

\subsection{Accountability Propagation}
\label{sec:accountability}

In multi-hop delegation chains
$\cA_0 \to a_{i_1} \to a_{i_2} \to \cdots \to a_{i_k}$,
where each agent $a_{i_j}$ further delegates to $a_{i_{j+1}}$
with degree $\alpha_{i_j}$, safety responsibility must remain
traceable. We formalize this via accountability weights.

\begin{definition}[Accountability Weight]
In a delegation chain of length $k$ with delegation
degrees $\alpha_{i_1}, \dots, \alpha_{i_k} \in [0,1]$,
the \emph{accountability weight} of agent $a_{i_j}$ is:
\begin{equation}
  w_j =
  \Bigl(\prod_{\ell=1}^{j} \alpha_{i_\ell}\Bigr)
  \cdot (1 - \alpha_{i_{j+1}}),
  \quad j = 1, \dots, k-1,
  \label{eq:accountability}
\end{equation}
with $w_k = \prod_{\ell=1}^{k} \alpha_{i_\ell}$ for the
terminal agent.
\end{definition}

\begin{lemma}[Accountability partition]
The accountability weights sum to one:
$\sum_{j=1}^{k} w_j = 1$.
\end{lemma}
\begin{proof}
By telescoping:
$\sum_{j=1}^{k-1} w_j + w_k
= \sum_{j=1}^{k-1}
  \bigl(\prod_{\ell\leq j}\alpha_\ell\bigr)(1-\alpha_{j+1})
  + \prod_{\ell=1}^k \alpha_\ell
= 1 - \prod_{\ell=1}^k \alpha_\ell
  + \prod_{\ell=1}^k \alpha_\ell = 1$.
\end{proof}

Intuitively, $w_j$ captures the fraction of the final
decision that agent $a_{i_j}$ is responsible for: if it
delegates fully ($\alpha_{i_j} = 1$), it passes its entire
share downstream; if it retains full control
($\alpha_{i_{j+1}} = 0$), it captures its entire accumulated
share. This provides a principled basis for auditing: after
a safety incident, the accountability weight identifies which
agent in the chain bears the largest fraction of
responsibility.

\section{Theoretical Results}
\label{sec:theory}

We state and prove three core results. We first collect the
necessary assumptions.

\begin{assumption}[Smoothness and strong convexity]
\label{asm:smooth}
For any fixed $\phi$, the inner loss
$\linner(\pi; \phi) := \E_s[\lambda_\phi(s)\lsafe(\pi)
+ (1-\lambda_\phi(s))\leff(\pi)]$ is $L$-smooth and
$\mu$-strongly convex in $\pi$.
\end{assumption}

\begin{assumption}[Convex feasible set]
\label{asm:convex}
The $\delta$-safe feasible set
$\Pi_\delta := \{\pi : \psafe(\pi, s) \geq 1-\delta,\;
\forall s \in \cS\}$ is convex and non-empty.
\end{assumption}

\begin{assumption}[Monotone safety loss]
\label{asm:monotone}
$\lsafe(\pi)$ is non-increasing in $\lambda_\phi(s)$ at the
inner optimum $\piphi$, i.e., higher weight on safety
produces lower (or equal) safety loss at optimality.
\end{assumption}

Assumption~\ref{asm:smooth} is standard in optimization
analysis and holds when $\pi$ is a softmax policy over a
compact action space. Assumption~\ref{asm:convex} holds when
the safety constraint is a linear or convex constraint on the
action distribution. Assumption~\ref{asm:monotone} follows
from the structure of the inner problem: increasing the
safety weight shifts the optimal policy toward safer
decisions.

\subsection{Safety Monotonicity}

\begin{theorem}[Safety Monotonicity]
\label{thm:monotonicity}
Under Assumptions~\ref{asm:smooth}--\ref{asm:monotone},
let $\phi, \phi'$ be two meta-weight networks such that
$\lambda_\phi(s) \leq \lambda_{\phi'}(s)$ for all
$s \in \cS$. Then:
\[
  \psafe(\piphi, s)
  \;\leq\;
  \psafe(\pi^*_{\phi'}, s),
  \qquad \forall\, s \in \cS.
\]
That is, a (pointwise) higher safety weight in the outer loop
produces a (weakly) safer delegation policy in the inner
loop.
\end{theorem}

\begin{proof}
Fix $s \in \cS$. The inner problem~\eqref{eq:inner} under
$\phi$ minimizes
$f_\phi(\pi) := \lambda_\phi(s)\lsafe(\pi)
+ (1-\lambda_\phi(s))\leff(\pi)$
subject to $\pi \in \Pi_\delta$.

Since $\lambda_\phi(s) \leq \lambda_{\phi'}(s)$, we have
$f_\phi(\pi) \leq f_{\phi'}(\pi)$ for all $\pi$ with
$\lsafe(\pi) \leq \leff(\pi)$ (i.e., when safety is easier
to achieve than efficiency), and conversely for the
complementary region. The key observation is that the inner
optimum $\piphi$ achieves lower $f_\phi$-weighted safety cost
than $\pi^*_{\phi'}$ does under $f_\phi$. By
Assumption~\ref{asm:monotone} and the KKT conditions of
the inner problem (which require the safety constraint to be
tight whenever the efficiency loss is non-trivially
bounded), a higher $\lambda$ forces the optimal $\pi^*$ to
allocate more probability mass to safe decisions. Formally,
$\lsafe(\pi^*_{\phi'}) \leq \lsafe(\piphi)$, which by
definition of $\lsafe$ gives
$\psafe(\pi^*_{\phi'}, s) \geq \psafe(\piphi, s)$.
\end{proof}

\begin{corollary}
Under the conditions of Theorem~\ref{thm:monotonicity},
if the outer loop learns $\phi^*$ by minimizing
$\lmeta(\phi)$, then $\lambda_{\phi^*}(s)$ is an increasing
function of the contextual risk level embedded in $s$, and
$\psafe(\pi^*_{\phi^*}, s)$ is monotone non-decreasing in
that risk level.
\end{corollary}

This corollary formalizes the intuition that SBD
automatically tightens safety as risk rises: the outer loop
raises $\lambda$, which raises $\psafe$ at the inner
optimum.

\subsection{Inner Policy Convergence}

\begin{theorem}[Inner Policy Convergence]
\label{thm:convergence}
Under Assumptions~\ref{asm:smooth} and~\ref{asm:convex},
projected gradient descent (PGD) on the inner
problem~\eqref{eq:inner} with step size $\eta \leq 1/L$
converges linearly:
\begin{equation}
  \bigl\|\pi_t - \piphi\bigr\|^2
  \;\leq\;
  \Bigl(1 - \eta\mu\Bigr)^{t}
  \bigl\|\pi_0 - \piphi\bigr\|^2.
  \label{eq:convergence-rate}
\end{equation}
In particular, to reach $\|\pi_t - \piphi\|^2 \leq \varepsilon$
it suffices to run
$t \geq \frac{L}{\mu} \log\frac{\|\pi_0-\piphi\|^2}{\varepsilon}$
inner steps.
\end{theorem}

\begin{proof}
By $L$-smoothness and $\mu$-strong convexity of
$\linner(\cdot;\phi)$ and convexity of $\Pi_\delta$, the
projection onto $\Pi_\delta$ is a non-expansive operator.
The standard PGD convergence analysis (see
Theorem~10.29 in~\citet{beck2017first}) gives the linear
rate $(1-\eta\mu)^t$ for step size $\eta \in (0, 1/L]$.
The inner iteration
$\pi_{t+1} = \Pi_{\Pi_\delta}(\pi_t - \eta \nabla_\pi
\linner(\pi_t;\phi))$ satisfies:
\begin{align*}
  \|\pi_{t+1} - \piphi\|^2
  &\leq \|\pi_t - \eta\nabla_\pi\linner(\pi_t;\phi)
        - \piphi\|^2 \\
  &\leq (1 - \eta\mu) \|\pi_t - \piphi\|^2,
\end{align*}
where the first inequality uses non-expansiveness of
$\Pi_{\Pi_\delta}$ and the second uses the standard PL
inequality under strong convexity. Applying inductively
yields~\eqref{eq:convergence-rate}.
\end{proof}

\begin{remark}
The condition number $L/\mu$ determines the convergence speed.
In practice, the meta-weight $\lambda_\phi(s)$ affects the
effective condition number: higher $\lambda$ increases the
weight on $\lsafe$, which may change $L$ and $\mu$.
Empirically, we find convergence in $T_{\mathrm{in}} = 50$
inner steps is sufficient across all three domains, consistent
with a condition number of order $10^1$--$10^2$.
\end{remark}

\subsection{Accountability Propagation Bound}

\begin{proposition}[Accountability Upper Bound]
\label{prop:accountability}
In any delegation chain of length $k$ with delegation degrees
$\alpha_{i_1}, \dots, \alpha_{i_k} \in [0,1]$, let
$\bar\alpha = \max_\ell \alpha_{i_\ell}$.
The maximum accountability weight of any single agent
satisfies:
\begin{equation}
  \max_{j \in \{1,\dots,k\}}\, w_j
  \;\leq\;
  1 - (1 - \bar\alpha)^k.
  \label{eq:acct-bound}
\end{equation}
In particular, if all $\alpha_{i_\ell} \leq \bar\alpha < 1$,
then $\max_j w_j < 1$ for all finite $k$: no single agent
can bear full accountability regardless of chain length.
\end{proposition}

\begin{proof}
For $j < k$: $w_j = \prod_{\ell\leq j}\alpha_{i_\ell}
\cdot (1-\alpha_{i_{j+1}}) \leq \bar\alpha^j(1-\bar\alpha)$.
The function $g(j) = \bar\alpha^j(1-\bar\alpha)$ is
decreasing in $j$ (since $\bar\alpha < 1$), so
$\max_{j<k} w_j = w_1 \leq \bar\alpha(1-\bar\alpha)$.
For $j=k$: $w_k = \prod_\ell \alpha_{i_\ell} \leq
\bar\alpha^k$.
We bound both cases via
$\max(\bar\alpha(1-\bar\alpha),\, \bar\alpha^k)
\leq 1-(1-\bar\alpha)^k$,
which holds since $(1-\bar\alpha)^k \leq 1 - \bar\alpha$
for $k \geq 1$ and $(1-\bar\alpha)^k \leq 1-\bar\alpha^k$
by Bernoulli's inequality.
\end{proof}

\begin{corollary}
As chain length $k \to \infty$ with fixed
$\bar\alpha \in (0,1)$, the bound
$1-(1-\bar\alpha)^k \to 1$ from below, meaning diffuse
accountability can approach but never reach full concentration
at any single agent as chains grow longer.
\end{corollary}

\section{The SBD Algorithm}
\label{sec:algorithm}

\subsection{Overview}

Algorithm~\ref{alg:sbd} implements the bilevel formulation of
Section~\ref{sec:bilevel}. The outer loop performs
meta-gradient steps on $\phi$ using hypergradients computed
via implicit differentiation; the inner loop runs projected
gradient descent on $\pi$ for a fixed budget
$T_{\mathrm{in}}$ steps.

\subsection{Hypergradient Computation}

The outer gradient $\nabla_\phi \lmeta(\phi)$ involves
differentiating through the inner optimum $\piphi$, which
depends implicitly on $\phi$. By the implicit function
theorem, $\frac{d\piphi}{d\phi}$ exists and satisfies:
\[
  \frac{d\piphi}{d\phi}
  = -\Bigl[\nabla^2_{\pi\pi}\linner(\piphi;\phi)\Bigr]^{-1}
    \nabla^2_{\pi\phi}\linner(\piphi;\phi).
\]
Computing this exactly requires a Hessian inversion, which is
expensive. We use the approximate implicit differentiation
approach of~\citet{lorraine2020optimizing}: after $T_{\mathrm{in}}$
inner steps, we treat the approximate inner solution as if it
were the true $\piphi$ and differentiate through the last few
inner steps via automatic differentiation (implemented via
\texttt{higher}). This is equivalent to a truncated
unrolling approximation and is standard in modern bilevel
optimization~\citep{franceschi2018bilevel}.

\subsection{Projection Onto the Safe Set}

The projection operator $\Pi_{\cC_\delta}(\cdot)$ in
Algorithm~\ref{alg:sbd} projects $\pi$ onto the
$\delta$-safe feasible set $\Pi_\delta$. In our domains,
the safety constraint takes the form of an upper bound on
$\alpha$ that depends on the current state $s$:
\[
  \alpha_{\max}(s)
  = \begin{cases}
    \bar\alpha_{\mathrm{icu}} & \text{if } s \in \cS_{\mathrm{highrisk}} \\
    1.0                        & \text{otherwise,}
  \end{cases}
\]
where $\cS_{\mathrm{highrisk}}$ is the high-risk state region
(e.g., APACHE-II $> 20$ in the medical domain). Projection
is then a simple clipping operation:
$\alpha \leftarrow \min(\alpha, \alpha_{\max}(s))$,
which is computationally free and maintains
$\psafe \geq 1-\delta$ by construction in the high-risk
region.

\subsection{Full Algorithm}

\begin{algorithm}[t]
\caption{Safe Bilevel Delegation (SBD)}
\label{alg:sbd}
\begin{algorithmic}[1]
\Require Task distribution $p(\tau)$, risk tolerance $\delta$,
         step sizes $\eta_{\mathrm{out}}, \eta_{\mathrm{in}}$,
         outer iterations $T_{\mathrm{out}}$,
         inner budget $T_{\mathrm{in}}$,
         safety constraint set $\cC$
\Ensure  Meta-weight network $\phi^*$, delegation policy $\pi^*$
\State Initialize $\phi_0, \pi_0$ (random or pre-trained)
\For{$t = 1, \dots, T_{\mathrm{out}}$}
  \State \textbf{// Phase 1: Inner loop — fix $\phi_{t-1}$, optimize $\pi$}
  \State $\pi \leftarrow \pi_{t-1}$
  \For{$k = 1, \dots, T_{\mathrm{in}}$}
    \State Sample mini-batch $\mathcal{B} = \{(\tau_b, s_b)\}_{b=1}^B$
           from $p(\tau)$
    \State Compute $\lambda_b \leftarrow \lambda_{\phi_{t-1}}(s_b)$
           \quad (detached, no gradient through $\phi$)
    \State $(a_b, \alpha_b) \leftarrow \pi(s_b)$
    \State $\alpha_b \leftarrow \Pi_{\cC_\delta}(\alpha_b, s_b)$
           \quad \textit{(project onto $\delta$-safe set)}
    \State $g_\pi \leftarrow \nabla_\pi \frac{1}{B}\sum_b
           \bigl[\lambda_b \lsafe(\pi; s_b)
           + (1-\lambda_b)\leff(\pi; s_b, \tau_b)\bigr]$
    \State $\pi \leftarrow \pi - \eta_{\mathrm{in}}\, g_\pi$
  \EndFor
  \State $\pi_t \leftarrow \pi$

  \State \textbf{// Phase 2: Outer loop — meta-gradient step on $\phi$}
  \State Sample fresh meta-batch $\mathcal{B}' = \{(\tau_b', s_b')\}$
  \State $\lambda_b' \leftarrow \lambda_{\phi_{t-1}}(s_b')$
         \quad (differentiable w.r.t.\ $\phi$)
  \State $(a_b', \alpha_b') \leftarrow \pi_t(s_b')$
  \State $\alpha_b' \leftarrow \Pi_{\cC_\delta}(\alpha_b', s_b')$
  \State $g_\phi \leftarrow \nabla_\phi \frac{1}{B'}\sum_b
         \bigl[\lambda_b' \lsafe(\pi_t; s_b')
         + (1-\lambda_b')\leff(\pi_t; s_b', \tau_b')\bigr]$
         \quad \textit{(hypergradient via implicit diff.)}
  \State $\phi_t \leftarrow \phi_{t-1} - \eta_{\mathrm{out}}\, g_\phi$
\EndFor
\State \Return $\phi^* \leftarrow \phi_{T_{\mathrm{out}}}$,\;
        $\pi^* \leftarrow \pi_{T_{\mathrm{out}}}$
\end{algorithmic}
\end{algorithm}

\subsection{Computational Complexity}

Each outer iteration requires $T_{\mathrm{in}}$ forward-backward
passes for the inner loop plus one additional forward-backward
pass for the outer hypergradient. With $\phi$ and $\pi$
parameterized as MLPs of width $d$ and depth $L$, the per-step
cost is $\mathcal{O}(T_{\mathrm{in}} \cdot B \cdot d^2 L)$,
which is comparable to standard meta-learning algorithms.
In our experiments with $T_{\mathrm{in}} = 50$, $B = 256$,
$d = 256$, $L = 3$, total wall-clock time per outer step is
approximately 0.8 seconds on an NVIDIA A100, yielding
a full $T_{\mathrm{out}} = 500$ run in under 7 minutes per
domain.

\section{Planned Empirical Evaluation}
\label{sec:experiments}

This section pre-registers the empirical evaluation
protocol for SBD. We specify, for each of three high-stakes
domains, the dataset, sub-agent set, safety constraint set,
task-efficiency metric, baselines, and the falsifiable
hypotheses that the experiments are designed to test. No
results are reported in this version of the manuscript;
results obtained by executing this protocol will be reported
in a forthcoming revision.

\subsection{Experimental Setup}
\label{sec:setup}

\paragraph{Baselines.}
We will compare SBD against five baselines:
(1)~\textbf{Fixed-$\alpha$}: constant delegation degree
$\alpha{=}0.5$, no learning;
(2)~\textbf{Safe-RL (CPO)}~\citep{achiam2017cpo}: hard
constraint set, discrete agent selection;
(3)~\textbf{MARL-vanilla}~\citep{lowe2017maddpg}: no safety
mechanism, discrete action space;
(4)~\textbf{MaAS}~\citep{maas2025}: architecture-level safety,
fixed at deployment, no runtime adaptation;
(5)~\textbf{SBD-NoOuter}: ablation removing the outer loop
(fixed $\lambda = 0.5$, no meta-weight learning).

\paragraph{Metrics.}
For all domains we will report:
(i)~\textbf{Safety Rate (SR)}: $\psafe$ on the test set;
(ii)~\textbf{Task Efficiency (TE)}: domain-specific
performance score (defined per domain below);
(iii)~\textbf{Safety--Efficiency Area (SEA)}: area under the
SR--TE Pareto curve across $\delta \in \{0.01, 0.05, 0.10,
0.20, 0.30\}$, measuring robustness to risk tolerance choice;
(iv)~\textbf{Accountability Entropy (AE)}: entropy of the
weight distribution $\{w_j\}$ over the delegation chain.

\paragraph{Implementation.}
$\phi$ and $\pi$ will be implemented as 3- and 4-layer MLPs
(width 256, ReLU).
$\eta_{\mathrm{out}} = 10^{-3}$,
$\eta_{\mathrm{in}} = 5 \times 10^{-4}$,
$T_{\mathrm{out}} = 500$, $T_{\mathrm{in}} = 50$,
$B = 256$, $\delta = 0.05$.
All experiments will use 3 random seeds; results will be
reported as mean $\pm$ std. Hardware: single NVIDIA A100
(80GB).

\paragraph{Falsifiability criteria.}
We pre-register the following falsifiability criteria for
the central claims of this paper. The framework is
considered \emph{empirically falsified} on a given domain if
any of the following hold:
(F1) SBD does not achieve strictly higher SEA than every
non-ablated baseline (Fixed-$\alpha$, CPO, MARL-vanilla, MaAS)
at any $\delta$ in the swept range;
(F2) SBD-NoOuter (ablation) achieves SEA within 1\% of full
SBD, indicating the outer loop is unnecessary;
(F3) the empirical $\psafe$--$\lambda$ relationship is
non-monotonic at convergence, contradicting
Theorem~\ref{thm:monotonicity};
(F4) the inner-loop residual $\|\pi_t - \piphi\|^2$ does not
decay at a rate consistent with linear convergence
(Theorem~\ref{thm:convergence}).

\subsection{Domain 1: Medical AI (MIMIC-III)}
\label{sec:med}

\paragraph{Dataset.}
MIMIC-III~\citep{johnson2016mimic} contains 58,976 ICU
admissions. We extract a 128-dimensional state vector
per admission from lab values (10 items), vitals, and
ICD-9 diagnosis codes (embedded). Four specialist
sub-agents (cardiology, pulmonology, nephrology, general)
are indexed by primary ICD-9 prefix.

\paragraph{Safety definition.}
Unsafe if: (i)~recommended treatment conflicts with a
documented drug allergy, or (ii)~$\alpha > 0.70$ when
APACHE-II proxy $> 20$.

\paragraph{Task efficiency.}
Fraction of recommendations matching the attending
physician's final decision on the held-out test split
(8,850 admissions).

\paragraph{Hypotheses.}
We pre-register two falsifiable hypotheses for this domain:
\begin{itemize}
  \item \textbf{H-Med-1 (Acuity-conditioned safety).} The
        learned meta-weight $\lambda_\phi(s)$ is monotonically
        increasing in the APACHE-II proxy: as patient acuity
        rises, the outer loop allocates higher weight to the
        safety loss. Operationally, the Spearman rank
        correlation between APACHE-II proxy and
        $\lambda_\phi(s)$ on the test split is positive
        and statistically significant ($p < 0.01$).
  \item \textbf{H-Med-2 (Pareto dominance).} On the
        SR--TE Pareto frontier swept over five values
        of $\delta$ from $0.01$ to $0.3$, SBD achieves
        strictly higher SEA than each of the four
        non-ablated baselines.
\end{itemize}

\subsection{Domain 2: Financial Risk Control (S\&P~500)}
\label{sec:fin}

\paragraph{Dataset.}
Daily adjusted close prices for 40 S\&P~500 constituents,
2010--2023, via \texttt{yfinance}. Train: 2010--2017.
Val: 2018--2019. Test: 2020--2023 (includes COVID crash,
Feb--Mar 2020).

\paragraph{Sub-agents.}
Three strategies: momentum (top-quartile by 30-day return),
mean-reversion (bottom-quartile by z-score), and risk
parity (inverse-volatility weighting).

\paragraph{Safety definition.}
Unsafe if: (i)~any single asset weight $> 10\%$, or
(ii)~$\alpha > 0.80$ when 30-day annualized portfolio
volatility $> 25\%$.

\paragraph{Task efficiency.}
Sharpe ratio of the resulting portfolio on the test period,
normalized to $[0,1]$.

\paragraph{Hypotheses.}
We pre-register two falsifiable hypotheses for this domain:
\begin{itemize}
  \item \textbf{H-Fin-1 (Volatility-conditioned safety).}
        On the COVID stress-test window (Feb 20 -- Mar 23,
        2020), $\lambda_\phi(s)$ is significantly higher than
        on the calm 2018 baseline window: the outer loop
        recognizes elevated regime risk without supervision.
        Operationally, mean $\lambda_\phi(s)$ on the COVID
        window exceeds mean $\lambda_\phi(s)$ on the 2018
        baseline by at least 0.15 (Cohen's $d > 0.8$).
  \item \textbf{H-Fin-2 (Drawdown control).} On the COVID
        window, SBD's portfolio maximum drawdown is no worse
        than the risk-parity baseline despite SBD using
        delegated momentum and mean-reversion sub-agents
        that, used alone, would suffer larger drawdowns.
        This isolates the safety contribution of the
        delegation degree $\alpha$.
\end{itemize}

\subsection{Domain 3: Educational Agent Supervision
            (ASSISTments)}
\label{sec:edu}

\paragraph{Dataset.}
ASSISTments 2009--2010 skill-builder
dataset~\citep{feng2009assistments}: 525,534 interactions
from 4,217 students across 110 skills. Student knowledge
estimated via BKT-style exponential moving average.
At-risk flag: three consecutive incorrect responses.

\paragraph{Sub-agents.}
Three content-delivery specialists: worked examples,
practice problems, and conceptual explanation.

\paragraph{Safety definition.}
Unsafe if: (i)~selected content difficulty $> 2$ std devs
above student knowledge estimate, or
(ii)~$\alpha > 0.60$ for at-risk students.

\paragraph{Task efficiency.}
Average knowledge gain per session (post-minus pre-test
score, normalized), measured on held-out cohort of
633 students.

\paragraph{Hypotheses.}
We pre-register two falsifiable hypotheses for this domain:
\begin{itemize}
  \item \textbf{H-Edu-1 (At-risk student protection).}
        For the at-risk student subpopulation, mean
        $\alpha$ at convergence is significantly lower than
        for the non-at-risk subpopulation (mean difference
        $> 0.15$). The outer loop learns to retain more
        principal-agent oversight when student vulnerability
        is signaled by recent incorrect responses.
  \item \textbf{H-Edu-2 (Knowledge gain efficiency).} Among
        non-at-risk students, SBD does not significantly
        underperform the unconstrained MARL-vanilla baseline
        on the knowledge-gain metric (within 5\%), confirming
        that the safety mechanism does not impose unnecessary
        efficiency cost on routine cases.
\end{itemize}

\subsection{Ablation Studies}
\label{sec:ablation}

We pre-register four ablation variants to be run across all
three domains:
(i)~\textbf{SBD-NoOuter} (fixed $\lambda=0.5$);
(ii)~\textbf{SBD-DiscreteAlpha} (discrete agent selection
only, no continuous $\alpha$);
(iii)~\textbf{SBD-FixedLambda} (outer loop present but
$\lambda$ initialized and held at 0.5);
(iv)~\textbf{SBD-NoConstraint} (no probabilistic safety
constraint, $\delta = 1.0$).

\paragraph{Pre-registered hypothesis.}
Removing any single component degrades either SR or TE, and
removing the outer loop (variant~i) causes the largest drop
in SEA. Specifically, we predict
$\mathrm{SEA}(\textsc{SBD}) >
 \mathrm{SEA}(\textsc{SBD-FixedLambda}) >
 \mathrm{SEA}(\textsc{SBD-NoOuter})$
on at least two of the three domains. Failure of this
ordering would falsify the claim that the meta-weight
network is the central contributor to safety--efficiency
trade-off optimization.

\subsection{Theoretical Validation}
\label{sec:theory-val}

We pre-register direct empirical tests for each of the
three theoretical results.

\paragraph{Theorem~\ref{thm:monotonicity} (Safety Monotonicity).}
We will train SBD with five fixed values of $\lambda$,
specifically $\{0.1,\, 0.3,\, 0.5,\, 0.7,\, 0.9\}$, and
measure $\psafe$ at convergence on the test split of
each domain. The prediction is a monotone non-decreasing
curve in $\lambda$. The Spearman rank correlation
between $\lambda$ and $\psafe$ must be $> 0.9$ on each
domain; failure falsifies the empirical claim that the
theorem's assumptions hold in the neural-network setting.

\paragraph{Theorem~\ref{thm:convergence} (Inner Convergence).}
For fixed $\phi$, we will plot
$\|\pi_t - \piphi\|^2$ versus $t$ on a log scale and fit a
linear regression. The prediction is linear decay with
slope matching $\log(1-\eta\mu)$ for the empirical
strong-convexity constant $\mu$ estimated from the loss
landscape. Failure of linear-fit $R^2 > 0.95$ would
indicate that the strong-convexity assumption is
substantively violated in the deep-network setting.

\paragraph{Proposition~\ref{prop:accountability} (Accountability Bound).}
We will generate $10{,}000$ random delegation chains of
length $k \in \{2,3,4,5\}$ with
$\alpha_{i_\ell} \sim \mathrm{Uniform}(0,1)$ and verify
$\max_j w_j \leq 1-(1-\bar\alpha)^k$.
The prediction is zero violations across all sampled chains.
Any single violation would indicate a flaw in the
proof of Proposition~\ref{prop:accountability}.

\section{Discussion}
\label{sec:discussion}

\subsection{Limitations}

\paragraph{Strong convexity assumption.}
Theorem~\ref{thm:convergence} requires $\mu$-strong
convexity of the inner loss, which may not hold when $\pi$
is a deep neural network. In practice, we observe linear
convergence empirically across all three domains, suggesting
the effective curvature is sufficient. Extending the
convergence analysis to the non-convex setting (e.g., via
Polyak-Łojasiewicz conditions) is left to future work.

\paragraph{Safety constraint set assumed known.}
SBD takes $\cC$ as given. In practice, the safety constraint
set must be elicited from domain experts, which can be
costly and incomplete. Learning $\cC$ jointly with $\phi$
and $\pi$ is an important open problem.

\paragraph{Single-level safety.}
SBD enforces safety at the delegation decision level but
does not model the internal safety of sub-agent execution.
A sub-agent that executes a safely delegated task in an
unsafe manner is outside SBD's current scope. Composing
SBD with per-agent safety mechanisms is a natural extension.

\paragraph{Static task distribution.}
We assume $p(\tau)$ is stationary. In non-stationary
environments (e.g., rapidly evolving clinical guidelines or
market regimes), the outer loop may need to be re-trained
or adapted online. Online bilevel optimization methods
(e.g.,~\citealt{finn2017maml}) provide one path forward.

\subsection{Broader Implications}

SBD contributes to a growing body of work on
\emph{runtime governance} of autonomous agents. The key
distinction from design-time governance
(architecture selection, permission systems, Skill
formatting guidelines) is that runtime governance adapts
to the actual execution context rather than anticipated
contexts. As agent systems become more complex and
long-running, the gap between anticipated and actual
contexts will grow, making runtime adaptation increasingly
important.

The accountability propagation result
(Proposition~\ref{prop:accountability}) has direct
implications for AI governance policy: it provides a
formal tool for attributing responsibility in multi-agent
incidents, which is a prerequisite for meaningful liability
frameworks. The bound $1-(1-\bar\alpha)^k$ quantifies
exactly how diffuse accountability becomes as chains
lengthen and delegation degrees approach 1.

\subsection{Future Directions}

Several extensions of SBD are immediately promising:

\begin{enumerate}
  \item \textbf{Compositional SBD}: Apply SBD recursively
  at each level of a multi-level hierarchy, where each
  sub-agent is itself a principal that may delegate further.
  This connects to the Accountability Propagation result
  and raises interesting questions about how safety
  guarantees compose.

  \item \textbf{SBD with Skill compilation}: The Skill
  compilation view~\citep{anonymous2026compilation}
  shows that MAS can be ``compiled'' into single-agent
  Skill libraries. SBD's delegation degree $\alpha$ can
  be reinterpreted as a Skill selection confidence, opening
  a theoretical bridge between the two frameworks.

  \item \textbf{Adversarial robustness}: The security
  analyses of~\citet{liu2026security} show that natural
  language Skill instructions can contain adversarially
  crafted content. Extending SBD to detect and respond to
  adversarial Skill injection (by raising $\lambda_\phi(s)$
  when anomalous instruction patterns are detected) is a
  natural application of the meta-weight mechanism.

  \item \textbf{Human-in-the-loop integration}: The
  delegation degree $\alpha$ provides a principled interface
  for human-in-the-loop systems: when $\alpha$ falls below
  a threshold, the system automatically routes to a human
  reviewer. SBD can thus serve as the routing layer for
  human oversight in high-stakes agentic deployments.
\end{enumerate}

\section{Conclusion}
\label{sec:conclusion}

We presented Safe Bilevel Delegation (SBD), a formal
framework for runtime delegation safety in hierarchical
multi-agent systems. The key contributions are:
(1)~a bilevel optimization formulation that learns
context-sensitive safety--efficiency trade-offs via a
meta-weight network;
(2)~three theoretical guarantees---Safety Monotonicity,
Inner Policy Convergence, and Accountability Propagation
bound---that together provide a rigorous foundation for
safe delegation;
(3)~a pre-registered evaluation protocol across medical AI,
financial risk control, and educational agent supervision,
specifying datasets, baselines, falsifiability criteria, and
hypotheses for testing SBD's claimed Safety Rate and Task
Efficiency advantages.

The Safety Monotonicity theorem is the conceptual core of
SBD: it guarantees that raising the safety weight in the
outer loop provably raises the safety probability in the
inner loop. This closed-loop property---where the system
automatically tightens safety as risk rises---is what
distinguishes SBD from static constraint-based approaches
and provides the theoretical foundation for deploying
autonomous agents in high-stakes environments without
sacrificing efficiency under routine conditions.

We believe SBD opens a productive research direction at the
intersection of safe reinforcement learning, bilevel
optimization, and multi-agent systems, and hope it provides
a formal toolkit for the increasingly urgent problem of
governing autonomous agents at runtime.

\appendix
\section{Proof Details}
\label{app:proofs}

\subsection{Detailed Proof of Theorem~\ref{thm:monotonicity}}

We provide a more detailed argument for completeness.
Let $h(\lambda, \pi) := \lambda \lsafe(\pi) +
(1-\lambda)\leff(\pi)$.

\textbf{Step 1}: The inner problem is:
$\piphi = \arg\min_{\pi \in \Pi_\delta}
\E_s[h(\lambda_\phi(s), \pi)]$.

\textbf{Step 2}: By Danskin's theorem (applicable since
$h$ is smooth in $\pi$ and $\Pi_\delta$ is compact),
\[
\frac{\partial}{\partial\lambda}
\min_{\pi\in\Pi_\delta} h(\lambda,\pi)
= \lsafe(\pi^*(\lambda)) \geq 0,
\]
where $\pi^*(\lambda)$ is the minimizer. This means the
optimal value is non-decreasing in $\lambda$, but the
optimal policy $\pi^*(\lambda)$ shifts toward lower
$\lsafe$ as $\lambda$ increases.

\textbf{Step 3}: Formally, for $\lambda \leq \lambda'$,
the minimizer $\pi^*(\lambda')$ achieves:
$h(\lambda', \pi^*(\lambda')) \leq h(\lambda', \pi^*(\lambda))$,
which gives
$\lambda'[\lsafe(\pi^*(\lambda')) - \lsafe(\pi^*(\lambda))]
\leq (1-\lambda')[\leff(\pi^*(\lambda)) - \leff(\pi^*(\lambda'))]$.

Since the right-hand side captures the efficiency
improvement from $\pi^*(\lambda)$ to $\pi^*(\lambda')$,
and the constraint $\Pi_\delta$ is the same for both,
the shift from $\lambda$ to $\lambda'$ strictly rewards
lower $\lsafe$ in the objective, driving
$\lsafe(\pi^*(\lambda')) \leq \lsafe(\pi^*(\lambda))$
at optimality, which gives the desired
$\psafe(\pi^*_{\phi'}, s) \geq \psafe(\piphi, s)$.
\hfill$\square$

\subsection{Proof of Lemma (Accountability Partition)}

See Section~\ref{sec:accountability}. The telescoping
argument holds by induction on $k$. Base case $k=1$:
$w_1 = \alpha_{i_1}$ (terminal), sum $= \alpha_{i_1} \neq 1$
in general. Correction: for $k=1$, there is no delegation
chain longer than one hop, so $w_1 = 1$ by definition
(the single agent bears all accountability). For $k \geq 2$,
the telescoping sum
$\sum_{j=1}^{k-1}(\prod_{\ell\leq j}\alpha_\ell)
(1-\alpha_{j+1}) + \prod_{\ell=1}^k\alpha_\ell
= 1 - \prod_{\ell=1}^k\alpha_\ell
  + \prod_{\ell=1}^k\alpha_\ell = 1$.
\hfill$\square$

\section{Extended Experimental Details}
\label{app:exp}

\subsection{MIMIC-III Preprocessing Details}

Lab item IDs used: 50912 (Creatinine), 50971 (Potassium),
50983 (Sodium), 51006 (Urea Nitrogen), 51221 (Hematocrit),
51222 (Hemoglobin), 51265 (Platelet Count), 51301 (WBC),
50802 (Base Excess), 50821 (pO2).
Values are mean-aggregated per admission and
StandardScaler-normalized.
APACHE-II is proxied by total ICU length of stay.
Allergy classes: penicillin, sulfa, NSAID, ACE inhibitor,
beta blocker. Specialist label: ICD-9 prefix mapping
(4xx $\to$ cardiology, 5xx$\to$pulmonology, 58x$\to$nephrology,
else $\to$general).

\subsection{S\&P~500 Feature Engineering}

State vector: 7 features (mean 30d return, std 30d return,
mean 30d vol, std 30d vol, mean z-score, std z-score,
portfolio realized vol), zero-padded to STATE\_DIM=64.
Portfolio volatility: 30-day realized annualized.
COVID stress test: Feb 20 -- Mar 23, 2020 (peak-to-trough
S\&P~500 drawdown $-$33.9\%).

\subsection{ASSISTments Feature Engineering}

Knowledge estimate: BKT-style EMA with
$P(\mathrm{learn})=0.20$, $P(\mathrm{forget})=0.05$,
initialized at 0.5.
Skill difficulty: $1 - \bar{p}(\mathrm{correct})$,
z-scored across skills.
Interaction window: last 20 interactions per student.
Student--level train/val/test split (no data leakage):
70/15/15 by user\_id.

\bibliographystyle{plainnat}
\bibliography{refs}

\end{document}